\PassOptionsToPackage{unicode}{hyperref}
\PassOptionsToPackage{hyphens}{url}
\PassOptionsToPackage{dvipsnames,svgnames,x11names}{xcolor}
\documentclass[12pt]{article}

\usepackage{amsmath,amssymb,amsthm}
\usepackage[T1]{fontenc}
\usepackage[utf8]{inputenc}
\usepackage{textcomp}
\usepackage{lmodern}
\IfFileExists{microtype.sty}{\usepackage{microtype}}{}
\usepackage{xcolor}
\usepackage{booktabs,array,multirow}
\usepackage{enumitem}
\usepackage{graphicx}
\usepackage{algorithm}
\usepackage{algpseudocode}
\usepackage{float}
\usepackage[hidelinks]{hyperref}
\usepackage{bookmark}
\usepackage{parskip} 

\usepackage{natbib}
\newtheorem{theorem}{Theorem}
\newtheorem{proposition}[theorem]{Proposition}
\newtheorem{corollary}[theorem]{Corollary}

\newtheorem{remark}[theorem]{Remark}
\newtheorem{definition}[theorem]{Definition}
\newtheorem{assumption}{Assumption}

\newcommand{\R}{\mathbb{R}}

\newcommand{\osc}{\mathrm{osc}}

\newcommand{\anon}{1}

\begin{document}

\def\spacingset#1{\renewcommand{\baselinestretch}{#1}\small\normalsize}
\spacingset{1}

\if1\anon
{
  \title{\bf Self-Certifying Transport MCMC via Dual Spectral-Gap Certificates}
  \author{Jun Hu\\
    Wuhan University of Technology, Wuhan, China}
  \date{}
  \maketitle
} \fi

\if0\anon
{
  \bigskip\bigskip\bigskip
  \begin{center}
    {\LARGE\bf Self-Certifying Transport MCMC via Dual Spectral-Gap Certificates}
  \end{center}
  \medskip
} \fi

\bigskip
\begin{abstract}
We propose CerT-MCMC, a framework that equips learned-transport Markov chain Monte Carlo with automatic, rigorous convergence certificates.
A normalising flow maps a Gaussian reference to an approximation of the target posterior; the same flow then serves as both the independence Metropolis--Hastings proposal and the basis for a computable spectral-gap bound.
We develop two complementary certificates.
The \emph{covering certificate} bounds the weight-ratio oscillation over the full proposal support via finite-sample covering arguments, yielding full-support spectral-gap bounds when a conservative gradient bound is available; its correction term scales as $O(n^{-1/D})$, making it rapidly weak and eventually vacuous as dimension increases.
The \emph{quantile-core certificate} restricts attention to a high-probability residual core on which the oscillation is controlled by one-dimensional empirical quantiles, with a finite-sample probability slack of $O(n^{-1/2})$, independent of the ambient dimension.
On synthetic targets ($D = 2$--$20$), structural-engineering posteriors ($D = 6,8$), real-data logistic regression on the Heart Disease data set ($D = 13$), and synthetic Bayesian logistic regression ($D = 20$), the quantile-core certificate delivers non-vacuous spectral-gap bounds where the covering certificate is vacuous, and its spectral-gap proxy tracks empirical effective sample sizes within $7\%$.
A negative control experiment confirms that the certificate discriminates flow quality by a factor exceeding $10\times$, whereas acceptance rates differ by only $1.15\times$.
To our knowledge, the dual-certificate framework is the first to provide automatic, dimension-aware convergence certificates for learned-transport MCMC, distinguishing genuine transport failure from proof-technique limitations.
\end{abstract}

\noindent%
\vfill

\newpage
\spacingset{1.8}

\section{Introduction}\label{sec:intro}

Markov chain Monte Carlo (MCMC) is the workhorse of Bayesian computation, yet verifying that a chain has converged remains an open problem.
Practitioners rely on heuristic diagnostics---trace plots, the $\hat{R}$ statistic \citep{GelmanRubin1992, BrooksGelman1998, Vehtari2021}, effective sample size estimates \citep{Vats2019}---that can detect certain pathologies but cannot certify convergence.
A spectral-gap bound, by contrast, provides a \emph{certificate}: a computable quantity that guarantees the chain mixes within a known number of iterations.
Such certificates exist in theory \citep{MengersenTweedie1996, Rosenthal1995, Tierney1994} but have rarely been turned into practical tools, because they require tight control of quantities---such as the importance-weight oscillation---that depend on the unknown target distribution.

As Bayesian workflows increasingly rely on learned, adaptive, and amortised proposals, convergence assessment becomes not easier but harder: the transition kernel itself is partly learned and high-dimensional, making \emph{auditability} a central requirement for modern Bayesian computation.

A recent line of work has shown that \emph{normalising flows} \citep{Rezende2015, Dinh2017, Papamakarios2021} can serve as powerful MCMC proposals.
By training an invertible neural network $T_\phi$ to map a simple reference distribution (typically a Gaussian) to an approximation of the target, one obtains an adaptive, amortised proposal for independence Metropolis--Hastings (IMH).
Existing methods focus on using this proposal to improve empirical mixing: \citet{ParnoMarzouk2018} use transport maps to accelerate MCMC in moderate dimensions; \citet{Hoffman2019NeuTra} apply normalising flows to precondition Hamiltonian Monte Carlo; \citet{Gabrie2022} and \citet{Wong2023flowMC} develop adaptive flow-based samplers that demonstrate strong empirical efficiency.
All of these methods preserve the correct stationary distribution via the Metropolis--Hastings accept/reject step, but none provides a computable convergence guarantee.
The central question of this paper is different from that of the transport-MCMC literature:

\begin{quote}
\emph{Prior work asks: can a learned proposal improve mixing?\\
We ask: can a learned proposal certify its own convergence?}
\end{quote}

We show that the answer is yes in two restricted but operationally useful senses, each confronting the fundamental tension between the dimensionality of the parameter space and the finite-sample cost of certification.

\paragraph{The certification route.}
When the proposal $q_\phi$ is used in an IMH kernel targeting $\pi$, the acceptance probability depends on the importance-weight ratio $w(z) = \pi_z(z)/q_0(z)$, where $\pi_z$ is the pullback of $\pi$ to the latent $z$-space.
The Mengersen--Tweedie theorem \citep{MengersenTweedie1996} states that the spectral gap of the IMH kernel is at least $2/(1 + e^C)$, where $C = \osc(\log w)$ is the oscillation of the log-weight over the proposal support.
If $T_\phi$ is a perfect transport map, $\log w$ is constant and $C = 0$; for an imperfect flow, $C$ measures the quality of the learned transport.
Crucially, the residual $r(z) = \log w(z)$ is computable for any $z$ given $T_\phi$ and $\pi$, without knowing the normalising constant of $\pi$.
This opens a path to data-driven certification: draw samples from the proposal, evaluate $r$ at each sample, and bound the oscillation from the observed values.

\paragraph{Certificate~I: covering the full support.}
The most direct approach bounds $\osc(r)$ by the sample range plus a Lipschitz correction for the gap between sample points.
Given $n$ certification samples in $D$ dimensions, the worst-case gap scales as $\varepsilon^* \sim n^{-1/D}$---the classical curse of dimensionality for covering arguments.
For $D = 2$ or $5$, this yields non-vacuous certificates at practical sample sizes; for $D \gtrsim 6$, the correction dominates and the certificate rapidly weakens; by $D = 20$ it is effectively vacuous.
We call this the \emph{covering certificate} (Certificate~I) and develop it formally in Section~\ref{sec:covering}.

\paragraph{Certificate~II: the quantile-core certificate.}
The covering certificate's failure in moderate dimensions is not a failure of the sampler.
Empirically, a well-trained flow produces a residual $r(z)$ that is highly concentrated: at $D = 10$, the central $99\%$ interquantile range is only $11\%$ of the full range (Table~\ref{tab:banana-quantiles}).
Certificate~I pays an exponentially growing price to control the remaining $1\%$ of boundary outliers.
Certificate~II avoids this entirely by restricting attention to a \emph{data-defined residual core}---the set of $z$-values whose residuals fall within the central quantile range---and certifying the IMH kernel restricted to this core.
The key insight is that the certified oscillation on the core depends only on one-dimensional quantile estimation, whose finite-sample probability slack is $O(n^{-1/2})$ via the Dvoretzky--Kiefer--Wolfowitz inequality, \emph{independent of the ambient dimension~$D$}.
Moreover, the normalising constants in the restricted importance weight cancel exactly (Section~\ref{sec:quantile}), so the Mengersen--Tweedie theorem applies directly to the core-restricted kernel.
The result is a computable, high-confidence spectral-gap lower bound that remains non-vacuous through $D = 20$ on the banana target and on real structural-engineering posteriors where Certificate~I is completely vacuous.
In effect, CerT-MCMC changes the role of learned transport from a black-box accelerator into an auditable object of Bayesian computation: the same map that proposes moves also exposes residual geometry from which finite-sample convergence certificates can be constructed.

\paragraph{Dual-certificate framework.}
Certificates~I and~II are complementary.
Certificate~I provides a full-support worst-case guarantee wherever the covering argument is tractable.
Certificate~II provides a high-mass core guarantee in all dimensions, honestly reporting the uncertified tail fraction~$2\rho$.
Together they yield a four-level diagnostic taxonomy (Figure~\ref{fig:hero}d): \emph{full certificate success} (both certificates non-vacuous), \emph{core certificate success} (Certificate~II non-vacuous, Certificate~I vacuous---the flow is good but the covering proof is the bottleneck), \emph{transport degradation} (Certificate~II non-vacuous but weak, reflecting genuine flow quality decline), and \emph{transport failure} (Certificate~II vacuous, indicating a poor flow).
This taxonomy distinguishes proof-technique limitations from genuine transport failure---a distinction that standard convergence diagnostics do not directly provide.

\paragraph{Contributions.}
\begin{enumerate}[nosep, label=(\roman*)]
\item We propose CerT-MCMC, a framework that equips learned-transport MCMC with automatic, rigorous spectral-gap certificates derived from the same normalising flow used as the IMH proposal.
\item We develop the \emph{covering certificate} (Certificate~I), prove a worst-case $\Omega(n^{-1/D})$ lower bound for pointwise full-support Lipschitz certification (Proposition~\ref{prop:lower}), and thereby formalise the covering barrier that motivates Certificate~II.
\item We develop the \emph{quantile-core certificate} (Certificate~II), which bounds the core oscillation via empirical quantiles with a finite-sample probability slack of $O(n^{-1/2})$ independent of $D$, and prove a spectral-gap theorem for the core-restricted IMH kernel.
\item We demonstrate empirically that the spectral-gap proxy from Certificate~II tracks actual mixing efficiency: on Bayesian logistic regression ($D = 20$), the proxy at a fixed trimming level matches the observed effective sample size within $7\%$.
\item We provide a comprehensive experimental evaluation on synthetic targets ($D = 2$--$20$), structural-engineering posteriors ($D = 6, 8$), synthetic Bayesian logistic regression ($D = 20$), and real-data Heart Disease logistic regression ($D = 13$), showing that Certificate~II delivers non-vacuous guarantees wherever Certificate~I is vacuous and correctly identifies transport quality differences across targets.
\end{enumerate}

\paragraph{Organisation.}
Section~\ref{sec:framework} introduces the transport-map framework, the Mengersen--Tweedie bound, and Certificate~I.
Section~\ref{sec:quantile} develops the quantile-core certificate (Certificate~II) with full proofs.
Section~\ref{sec:ess-theory} relates the core certificate to mixing efficiency.
Section~\ref{sec:experiments} presents experiments.
Section~\ref{sec:discussion} discusses limitations and future directions.

\section{Framework}\label{sec:framework}

\subsection{Transport map and independence Metropolis--Hastings}\label{sec:transport}

Let $\pi(\theta)$ denote the (possibly unnormalised) target density on $\Theta \subseteq \R^D$.
A normalising flow $T_\phi : \R^D \to \Theta$ with learnable parameters $\phi$ defines a pushforward density
\[
  q_\phi(\theta) = q_0\bigl(T_\phi^{-1}(\theta)\bigr) \,\bigl|\det J_{T_\phi^{-1}}(\theta)\bigr|,
\]
where $q_0 = \mathcal{N}(0, I_D)$ is the standard Gaussian reference.

We use $q_\phi$ as the proposal in an independence Metropolis--Hastings (IMH) kernel targeting $\pi$.
Working in the latent $z$-space via $\theta = T_\phi(z)$, the pullback target density is $\pi_z(z) = \pi(T_\phi(z))\,|\det J_{T_\phi}(z)|$, and the IMH acceptance probability from state $z$ to proposal $z'$ is
\[
  a(z, z') = \min\!\Bigl(1,\; \frac{w(z')}{w(z)}\Bigr), \qquad w(z) = \frac{\pi_z(z)}{q_0(z)} = e^{r(z)},
\]
where the \emph{log-weight residual} is
\begin{equation}\label{eq:residual}
  r(z) = \log \pi\bigl(T_\phi(z)\bigr) + \log\bigl|\det J_{T_\phi}(z)\bigr| + \tfrac{1}{2}\|z\|^2 + \tfrac{D}{2}\log(2\pi).
\end{equation}
If $T_\phi$ were a perfect transport map, $r(z)$ would be constant.
The oscillation of $r$ measures the quality of the learned transport.

\subsection{Truncated proposal and the Mengersen--Tweedie bound}\label{sec:mt}

Fix $\alpha \in (0,1)$ and define $K_\alpha = \{z : \|z\| \le R_\alpha\}$ with $R_\alpha = \sqrt{\chi^2_D(1-\alpha)}$, so that $q_0(K_\alpha) = 1-\alpha$.
The $K_\alpha$-restricted IMH kernel uses proposal $\mu_K = q_0(\cdot \mid K_\alpha)$ and targets $\pi_z(\cdot \mid K_\alpha)$.

By the Mengersen--Tweedie theorem \citep{MengersenTweedie1996}, the spectral gap of this kernel in $L^2(\pi_z(\cdot | K_\alpha))$ satisfies
\begin{equation}\label{eq:mt}
  \gamma \;\ge\; \frac{2}{1 + \exp\!\bigl(\osc_{K_\alpha}(r)\bigr)},
  \qquad \osc_{K_\alpha}(r) = \sup_{z \in K_\alpha} r(z) - \inf_{z \in K_\alpha} r(z).
\end{equation}
The certification task reduces to bounding $\osc_{K_\alpha}(r)$ from data.

\subsection{Covering certificate (Certificate~I)}\label{sec:covering}

The covering certificate bounds $\osc_{K_\alpha}(r)$ directly from a finite design of certification samples.
Given $n$ samples drawn over $K_\alpha$, the observed sample oscillation $\max_i r_i - \min_i r_i$ underestimates the true supremum because no finite design hits the exact extremal points; we therefore add a Lipschitz correction $2\,\|\nabla r\|_{\sup}\,\varepsilon^*$ that accounts for the residual variation between sample points and their nearest unobserved neighbours.
The correction radius $\varepsilon^*$ is the worst-case covering gap of $n$ points in the $D$-dimensional ball $K_\alpha$, obtained from a covering-number argument on $K_\alpha$, which is why $\varepsilon^*$ scales as $n^{-1/D}$.
Combining these two terms yields the computable bound $C_{\mathrm{I}} \ge \osc_{K_\alpha}(r)$, which plugs into the Mengersen--Tweedie bound~\eqref{eq:mt} to give a rigorous full-support spectral-gap guarantee.

Given $n$ certification samples $z_1, \ldots, z_n \sim \nu_K$, where $\nu_K = \mathrm{Unif}(K_\alpha)$ is a space-filling design on $K_\alpha$, with $r_i = r(z_i)$:
\begin{equation}\label{eq:v1-cert}
  C_{\mathrm{I}} = \max_i r_i - \min_i r_i + 2\,\|\nabla r\|_{\sup}\,\varepsilon^*,
  \qquad \varepsilon^* = R_\alpha \Bigl(\frac{\log(2/\zeta)}{2n}\Bigr)^{1/D},
\end{equation}
where $\|\nabla r\|_{\sup}$ is a conservative upper bound on the gradient norm (in experiments, we use the empirical maximum over certification samples as a practical bound; formal validity requires a certified upper bound).
The correction $\varepsilon^*$ grows as $n^{-1/D}$, making $C_{\mathrm{I}}$ vacuous for $D \gtrsim 6$ at practical sample sizes.
Equation~\eqref{eq:v1-cert} is a simplified representative form that highlights the $n^{-1/D}$ scaling; a formal version incorporating logarithmic covering-number factors is given in the Supplementary Material (Section~A).

\begin{remark}[Covering barrier]
The $n^{-1/D}$ rate is intrinsic to covering arguments in $D$ dimensions and cannot be improved by local-Lipschitz refinements or boundary truncation; see the Supplementary Material (Section~A) for empirical confirmation.
\end{remark}

\begin{proposition}[Covering lower bound]\label{prop:lower}
For the class of $L$-Lipschitz functions on $K_\alpha$, any algorithm that queries $r$ at $n$ points $z_1,\ldots,z_n \in K_\alpha$ and outputs a bound $\hat{C} \ge \osc_{K_\alpha}(r)$ valid for all $L$-Lipschitz $r$ must satisfy
\[
  \sup_{r \,:\, L\text{-Lip}} \bigl(\hat{C} - \osc_{K_\alpha}(r)\bigr) \;\ge\; c \, L \, R_\alpha \, n^{-1/D},
\]
for a dimension-dependent constant $c > 0$.
\end{proposition}

\begin{proof}
By a volume-packing argument, for any $n$ points in $K_\alpha = B(0, R_\alpha) \subset \mathbb{R}^D$, there exists $z^* \in K_\alpha$ with $\|z^* - z_i\| \ge \delta_n$ for all $i$, where $\delta_n \ge c \, R_\alpha \, n^{-1/D}$ for a constant $c > 0$ depending only on $D$.
Consider the two $L$-Lipschitz functions $r_0(z) \equiv 0$ and $r_1(z) = L\delta_n \max(0, 1 - \|z - z^*\|/\delta_n)$.
Both satisfy $r_0(z_i) = r_1(z_i) = 0$ for all $i$ (since $\|z^* - z_i\| \ge \delta_n$), so any point-evaluation-based algorithm cannot distinguish them.
Yet $\osc_{K_\alpha}(r_0) = 0$ while $\osc_{K_\alpha}(r_1) = L\delta_n$.
Any certificate valid for both must satisfy $\hat{C} \ge L\delta_n \ge c \, L \, R_\alpha \, n^{-1/D}$.
\end{proof}

Proposition~\ref{prop:lower} formalises the covering barrier: no algorithm based on pointwise evaluations can certify the full-support oscillation of a Lipschitz function faster than $\Omega(n^{-1/D})$, regardless of sample placement or local refinement. This lower bound applies to worst-case full-support Lipschitz certification; it does not preclude mass-aware or distributional certification strategies. This motivates the quantile-core construction of Certificate~II, which sidesteps covering entirely by working with the one-dimensional residual distribution.

\section{Quantile-Core Certificate (Certificate~II)}\label{sec:quantile}

The key observation is empirical: for a well-trained flow, the full oscillation $\osc_{K_\alpha}(r)$ is dominated by rare boundary outliers, while $r(z)$ is highly concentrated for the bulk of the proposal mass.
At $D = 10$, the $99\%$ interquantile range of $r$ is only $11\%$ of the full oscillation (Table~\ref{tab:banana-quantiles}).
Certificate~I pays an exponentially growing price to control the remaining $1\%$; Certificate~II avoids this entirely.

\subsection{Data-defined residual core}

Certificate~I requires a space-filling (e.g., uniform) design over $K_\alpha$ because it must cover the entire region.
Certificate~II only needs the one-dimensional distribution of $r(Z)$ under the \emph{proposal}; we therefore draw certification samples from the proposal distribution itself:
\begin{equation}\label{eq:cert-sampling}
  z_1, \dots, z_n \;\overset{\mathrm{iid}}{\sim}\; \mu_K \;:=\; q_0(\cdot \mid K_\alpha).
\end{equation}
This ensures that the quantiles estimated below are quantiles of $r(Z)$ under the IMH proposal, and the resulting mass statements refer directly to proposal mass.

\begin{definition}[Data-defined residual core]\label{def:core}
Fix trimming level $\rho \in (0, \tfrac{1}{2})$ and confidence level $\zeta \in (0,1)$.
Assume $\rho > \varepsilon_n$, where
\begin{equation}\label{eq:dkw-eps}
  \varepsilon_n = \sqrt{\frac{\log(2/\zeta)}{2n}};
\end{equation}
equivalently, $n > \log(2/\zeta)/(2\rho^2)$.
(All reported experiments satisfy this; for small $n$ or small $\rho$, the quantile endpoints should be clipped to $[0,1]$.)

Let $\hat{Q}_p$ denote the empirical $p$-quantile of the residuals $\{r_i = r(z_i)\}_{i=1}^n$.
Define the \emph{data-defined residual core}
\begin{equation}\label{eq:core-data}
  \tilde{G}_\rho = \bigl\{z \in K_\alpha : \hat{Q}_{\rho - \varepsilon_n} \le r(z) \le \hat{Q}_{1-\rho+\varepsilon_n}\bigr\},
\end{equation}
and the \emph{certified core oscillation}
\begin{equation}\label{eq:chat}
  \hat{C}_\rho = \hat{Q}_{1-\rho+\varepsilon_n} - \hat{Q}_{\rho - \varepsilon_n}.
\end{equation}
\end{definition}

By construction, $\osc_{\tilde{G}_\rho}(r) \le \hat{C}_\rho$ holds deterministically, since $\tilde{G}_\rho$ is defined by the empirical quantile thresholds.

\subsection{Main results}

\begin{assumption}\label{asm:continuous}
The distribution of $r(Z)$ under $Z \sim \mu_K$ is continuous, i.e., its CDF $F_r$ has no jumps.
\end{assumption}

This is a mild regularity condition satisfied generically when $T_\phi$ is smooth and the posterior density is continuous. It excludes degenerate cases where $r(Z)$ has atoms; a generalised-quantile formulation can remove this assumption at the cost of additional notation.

\begin{theorem}[Finite-sample quantile certificate]\label{thm:quantile}
Under Assumption~\ref{asm:continuous}, assume $\rho > \varepsilon_n$ (i.e., $n > \log(2/\zeta)/(2\rho^2)$).
Let $q_p = \inf\{t : F_r(t) \ge p\}$ denote the population $p$-quantile and $G_\rho = \{z \in K_\alpha : q_\rho \le r(z) \le q_{1-\rho}\}$ the true residual core.
Then with probability at least $1-\zeta$ over the certification sample:
\begin{equation}\label{eq:quantile-sandwich}
  \hat{Q}_{\rho - \varepsilon_n} \;\le\; q_\rho \;\le\; q_{1-\rho} \;\le\; \hat{Q}_{1-\rho+\varepsilon_n},
\end{equation}
and consequently $G_\rho \subseteq \tilde{G}_\rho$ and $\mu_K(\tilde{G}_\rho) \ge 1 - 2\rho$.
\end{theorem}

\begin{proof}
By the Dvoretzky--Kiefer--Wolfowitz inequality \citep{DKW1956,Massart1990}, with probability $\ge 1 - \zeta$,
\begin{equation}\label{eq:dkw-event}
  \sup_{t \in \R} |F_n(t) - F_r(t)| \le \varepsilon_n,
\end{equation}
where $F_n$ is the empirical CDF of $\{r_i\}$. All subsequent statements hold on this event.

\emph{Lower bound: $\hat{Q}_{\rho-\varepsilon_n} \le q_\rho$.}\;
By continuity of $F_r$ (Assumption~\ref{asm:continuous}), $F_r(q_\rho) = \rho$.
From~\eqref{eq:dkw-event}, $F_n(q_\rho) \ge F_r(q_\rho) - \varepsilon_n = \rho - \varepsilon_n$.
Since $\hat{Q}_{\rho-\varepsilon_n} = \inf\{t: F_n(t) \ge \rho - \varepsilon_n\}$ and $F_n(q_\rho) \ge \rho - \varepsilon_n$,
we have $\hat{Q}_{\rho-\varepsilon_n} \le q_\rho$.

\emph{Upper bound: $\hat{Q}_{1-\rho+\varepsilon_n} \ge q_{1-\rho}$.}\;
Suppose for contradiction that $\hat{Q}_{1-\rho+\varepsilon_n} < q_{1-\rho}$.
Set $t_0 = \hat{Q}_{1-\rho+\varepsilon_n}$.
By definition of the empirical quantile, $F_n(t_0) \ge 1-\rho+\varepsilon_n$.
From~\eqref{eq:dkw-event}, $F_r(t_0) \ge F_n(t_0) - \varepsilon_n \ge 1-\rho$.
But $t_0 < q_{1-\rho} = \inf\{t: F_r(t) \ge 1-\rho\}$ implies $F_r(t_0) < 1-\rho$, a contradiction.

\emph{Containment and mass.}\;
The sandwich~\eqref{eq:quantile-sandwich} gives $[q_\rho, q_{1-\rho}] \subseteq [\hat{Q}_{\rho-\varepsilon_n}, \hat{Q}_{1-\rho+\varepsilon_n}]$,
so $G_\rho \subseteq \tilde{G}_\rho$.
By continuity of $F_r$, $\mu_K(G_\rho) = F_r(q_{1-\rho}) - F_r(q_\rho) = 1 - 2\rho$, hence $\mu_K(\tilde{G}_\rho) \ge 1-2\rho$.
\end{proof}

\begin{theorem}[Core-restricted spectral gap]\label{thm:core-gap}
Let $\tilde{G}_\rho$ be the data-defined core from Definition~\ref{def:core}.
Consider the IMH kernel $P_{\tilde{G}}$ on $\tilde{G}_\rho$ with proposal $\mu_{\tilde{G}} = q_0(\cdot \mid \tilde{G}_\rho)$ and target $\pi_{\tilde{G}} = \pi_z(\cdot \mid \tilde{G}_\rho)$.
Then the spectral gap of $P_{\tilde{G}}$ in $L^2(\pi_{\tilde{G}})$ satisfies
\begin{equation}\label{eq:core-gap}
  \gamma(P_{\tilde{G}}) \;\ge\; \frac{2}{1 + e^{\hat{C}_\rho}}.
\end{equation}
\end{theorem}

\begin{proof}
The importance weight on $\tilde{G}_\rho$ is
\[
  w_{\tilde{G}}(z) = \frac{\pi_{\tilde{G}}(z)}{\mu_{\tilde{G}}(z)}
  = \frac{\pi_z(z)/\pi_z(\tilde{G}_\rho)}{q_0(z)/q_0(\tilde{G}_\rho)}
  = e^{r(z)} \cdot \frac{q_0(\tilde{G}_\rho)}{\pi_z(\tilde{G}_\rho)}.
\]
The second factor is constant in $z$, so $\osc_{\tilde{G}_\rho}(\log w_{\tilde{G}}) = \osc_{\tilde{G}_\rho}(r) \le \hat{C}_\rho$.
Apply the Mengersen--Tweedie bound~\eqref{eq:mt}.
The kernel certified here is the core-restricted IMH kernel on $\tilde{G}_\rho$, not the full $K_\alpha$-restricted chain; see Remark~\ref{rem:scope} for further discussion.
\end{proof}

\begin{corollary}[Certified core spectral gap from data]\label{cor:certified}
Under the conditions of Theorems~\ref{thm:quantile} and~\ref{thm:core-gap}, with probability at least $1-\zeta$:
\begin{enumerate}[nosep]
\item The data-defined core $\tilde{G}_\rho$ contains at least $1-2\rho$ of the proposal mass.
\item The IMH kernel restricted to $\tilde{G}_\rho$ has spectral gap $\gamma(P_{\tilde{G}}) \ge 2/(1 + e^{\hat{C}_\rho})$.
\item The finite-sample probability slack entering the quantile certificate is $\varepsilon_n = O(n^{-1/2})$, independent of the ambient dimension~$D$.
\end{enumerate}
\end{corollary}

\begin{remark}[Dimension-free rate]\label{rem:dim-free}
The certificate $\hat{C}_\rho$ depends on $D$ only through the population quantiles $q_\rho, q_{1-\rho}$ of $r(Z)$.
Its finite-sample probability slack is governed by $\varepsilon_n = O(n^{-1/2})$, which is independent of $D$.
In contrast, Certificate~I's correction scales as $O(n^{-1/D})$.
For $D = 20$ and $n = 200{,}000$: Certificate~I has $\varepsilon^* \approx 4.2$; Certificate~II has $\varepsilon_n \approx 0.003$.
Converting this probability slack into a quantile-value rate would require standard local density conditions near $q_\rho$ and $q_{1-\rho}$.
\end{remark}

\begin{remark}[Scope of the certificate]\label{rem:scope}
Theorem~\ref{thm:core-gap} certifies the spectral gap of the IMH chain \emph{restricted to $\tilde{G}_\rho$}, not the full $K_\alpha$-restricted chain.
The core $\tilde{G}_\rho$ captures at least $(1-2\rho)$ of the proposal mass and, when $\hat{C}_\rho$ is small, a comparable fraction of the target mass (Section~\ref{sec:target-mass}).
In Section~\ref{sec:ess-calibration} we demonstrate empirically that the core certificate tracks actual mixing efficiency.
\end{remark}

\subsection{Empirical target-mass diagnostic}\label{sec:target-mass}

The proposal-mass guarantee $\mu_K(\tilde{G}_\rho) \ge 1-2\rho$ is rigorous.
A corresponding target-mass statement requires controlling tail importance weights and is left as a formal guarantee for future work.
We report the empirical diagnostic
\begin{equation}\label{eq:pi-hat}
  \hat{\pi}(\tilde{G}_\rho) = \frac{\sum_{i=1}^n e^{r_i} \mathbf{1}_{\tilde{G}_\rho}(z_i)}{\sum_{i=1}^n e^{r_i}},
\end{equation}
which estimates $\pi_z(\tilde{G}_\rho \mid K_\alpha)$ via self-normalised importance sampling.
In well-matched cases (banana $D \le 10$, engineering targets), $\hat{\pi}(\tilde{G}_\rho)$ is close to $\mu_K(\tilde{G}_\rho)$ (Table~\ref{tab:target-mass}); in harder cases, the discrepancy serves as an empirical warning of residual transport mismatch.

\subsection{Dual-certificate framework}\label{sec:dual}

Certificates~I and~II serve complementary roles:

\begin{center}
\begin{tabular}{lcc}
\toprule
& \textbf{Certificate~I (Covering)} & \textbf{Certificate~II (Quantile Core)} \\
\midrule
Sampling design & Uniform on $K_\alpha$ & Proposal $q_0(\cdot \mid K_\alpha)$ \\
Region & Full $K_\alpha$ & Residual core $\tilde{G}_\rho$ \\
Controls & Worst-case oscillation & Core $(1-2\rho)$-mass oscillation \\
Finite-sample slack & $O(n^{-1/D})$ & $O(n^{-1/2})$ \\
Rigorous for & $D \lesssim 5$ & All $D$ (on the core) \\
\bottomrule
\end{tabular}
\end{center}

\noindent When both are available, Certificate~I provides the stronger (full-support) guarantee; Certificate~II extends certification to dimensions where Certificate~I is vacuous.

\begin{figure}[H]
\centering
\includegraphics[width=\textwidth]{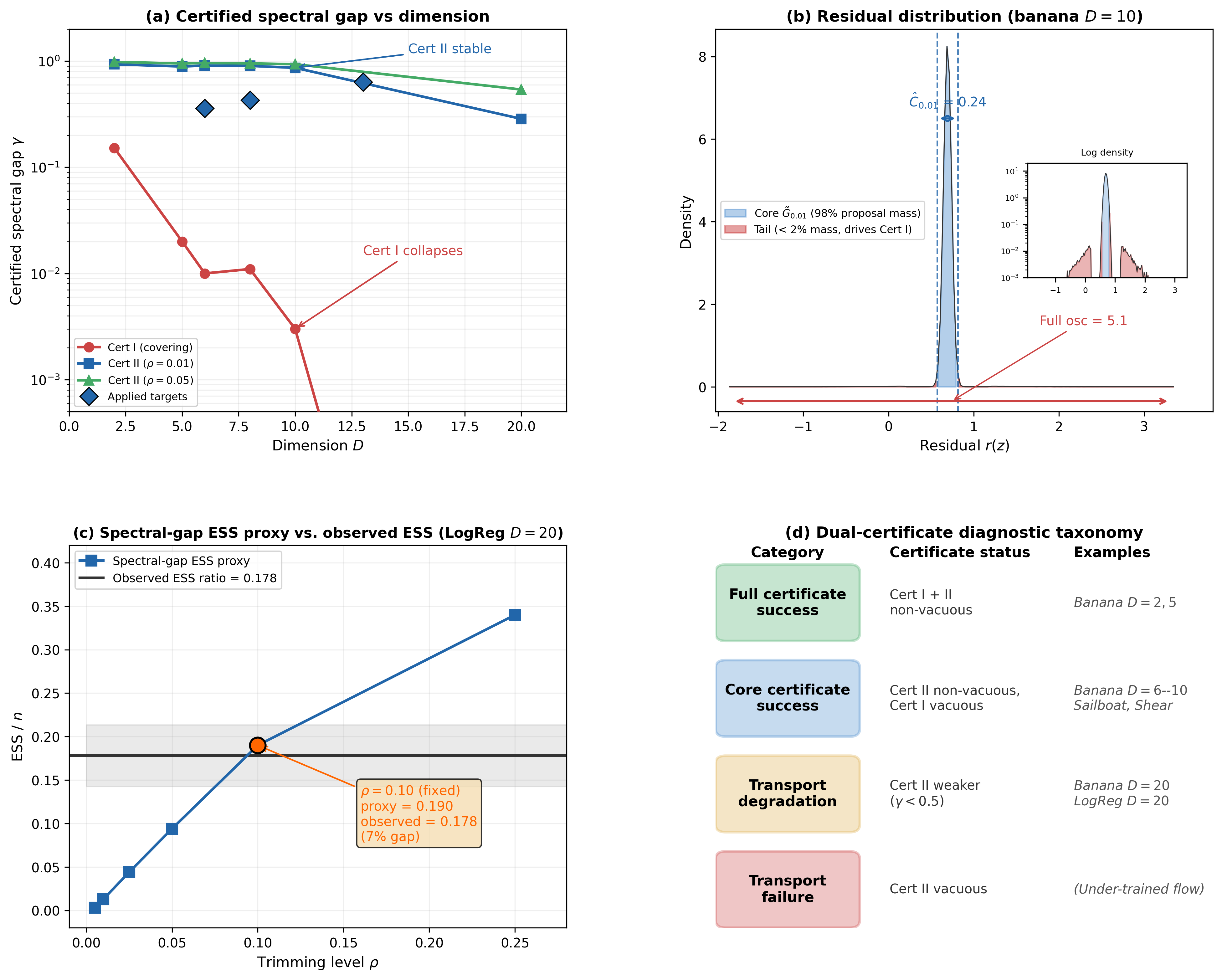}
\caption{Overview of the dual-certificate framework.
(a)~Certified spectral gap versus dimension on the banana target: Certificate~I (covering) collapses for $D \gtrsim 6$, whereas Certificate~II (quantile core) remains stable.
(b)~Residual distribution at $D = 10$, showing the concentrated core $\tilde{G}_{0.01}$ (capturing $98\%$ of the proposal mass) and the rare boundary tails that drive Certificate~I.
(c)~The spectral-gap ESS proxy from Certificate~II tracks the observed ESS ratio on Bayesian logistic regression ($D = 20$), matching within $7\%$ at $\rho = 0.10$.
(d)~The four-level diagnostic taxonomy induced by the two certificates.
Diamonds in~(a) denote additional applied and statistical targets (sailboat $D = 6$, shear-building benchmark $D = 8$, Heart Disease $D = 13$) evaluated using Certificate~II at $\rho = 0.01$. Panel~(c) is a calibration diagnostic at a pre-specified trimming level; it is not claimed as a universal ESS prediction.}
\label{fig:hero}
\end{figure}

\section{Relating Core Certificates to Mixing Efficiency}\label{sec:ess-theory}

For an independence MH kernel with spectral gap $\gamma$, the asymptotic variance of ergodic averages is bounded by $(2-\gamma)/\gamma$ times the target variance.
This motivates the \emph{spectral-gap ESS proxy}:
\begin{equation}\label{eq:ess-proxy}
  \widehat{\mathrm{ESS}}_\gamma / n \;=\; \frac{\gamma}{2 - \gamma},
\end{equation}
where $\gamma = \gamma(P_{\tilde{G}})$ from Theorem~\ref{thm:core-gap}.
The proxy~\eqref{eq:ess-proxy} is a conservative lower-bound-scale estimate of mixing efficiency, not a point prediction: since $\gamma(P_{\tilde{G}})$ certifies only the core-restricted kernel, the actual chain (which also mixes on the tails) may perform better.

The trimming level $\rho$ controls a bias--conservatism trade-off.
Small $\rho$ (e.g., $0.01$) certifies a wider core but includes more residual outliers, yielding a more conservative $\hat{C}_\rho$;
large $\rho$ (e.g., $0.25$) certifies a narrow core with a tighter $\hat{C}_\rho$ but covers less mass.
In Section~\ref{sec:ess-calibration}, we compare $\widehat{\mathrm{ESS}}_\gamma$ at fixed $\rho$ values with observed ESS from actual IMH chains, and identify the range of $\rho$ where the proxy tracks empirical mixing efficiency.

\section{Experiments}\label{sec:experiments}

All flows use the spectral-normalised RealNVP architecture trained with the CerT-OG-Anneal objective (NLL + smoothed oscillation penalty + gradient penalty with linear warm-up).
Certificate~I uses uniform certification samples on $K_\alpha$; Certificate~II uses proposal samples $\mu_K = q_0(\cdot \mid K_\alpha)$. Throughout, $\alpha = 0.05$.
Certificate~I values reported below are conservative diagnostic bounds under the empirical gradient estimate; their role is primarily to expose the covering barrier rather than to serve as production certificates.
Confidence level $\zeta = 0.05$ throughout.

\subsection{Banana target, $D = 2$--$20$}\label{sec:banana}

The $D$-dimensional banana target has potential $U(\theta) = \theta_1^2/(2\sigma_1^2) + (\theta_2 - b\theta_1^2)^2/(2\sigma_2^2) + \sum_{i \ge 3}\theta_i^2/2$, with $\sigma_1 = 2$, $\sigma_2 = 1$, $b = 0.1$.
This is a standard benchmark with nonlinear correlation structure in the first two dimensions and Gaussian marginals in the remainder.

\begin{table}[t]
\centering
\caption{Banana target: Certificate~I (covering) vs Certificate~II (quantile core) with $n = 200{,}000$ certification samples and $\rho = 0.01$.}
\label{tab:banana-quantiles}
\begin{tabular}{rrrrrrrr}
\toprule
$D$ & $\mathrm{Var}(r)$ & Full osc & $\hat{C}_{0.01}$ & $\gamma_{\mathrm{I}}$ & $\gamma_{\mathrm{II},\,0.01}$ & $\gamma_{\mathrm{II},\,0.05}$ & Ratio \\
\midrule
2  & 0.0004 & 2.27 & 0.13 & 0.152 & 0.933 & 0.982 & 0.05 \\
5  & 0.0013 & 2.91 & 0.22 & 0.020 & 0.890 & 0.950 & 0.07 \\
6  & 0.0011 & 2.50 & 0.19 & 0.010 & 0.908 & 0.961 & 0.06 \\
8  & 0.0012 & 2.11 & 0.20 & 0.011 & 0.902 & 0.952 & 0.08 \\
10 & 0.0022 & 2.27 & 0.27 & 0.003 & 0.865 & 0.935 & 0.11 \\
20 & 0.1226 & 54.65 & 1.79 & $\approx 0$ & 0.286 & 0.540 & 0.03 \\
\bottomrule
\end{tabular}

\smallskip
\begin{minipage}{\linewidth}
\small Ratio $= \hat{C}_{0.01} / \text{Full osc}$. Certificate~I is vacuous ($\gamma_{\mathrm{I}} \approx 0$) for $D \ge 20$; Certificate~II maintains $\gamma > 0.85$ through $D = 10$ and remains non-vacuous at $D = 20$.
\end{minipage}
\end{table}

\begin{figure}[H]
\centering
\includegraphics[width=0.95\textwidth]{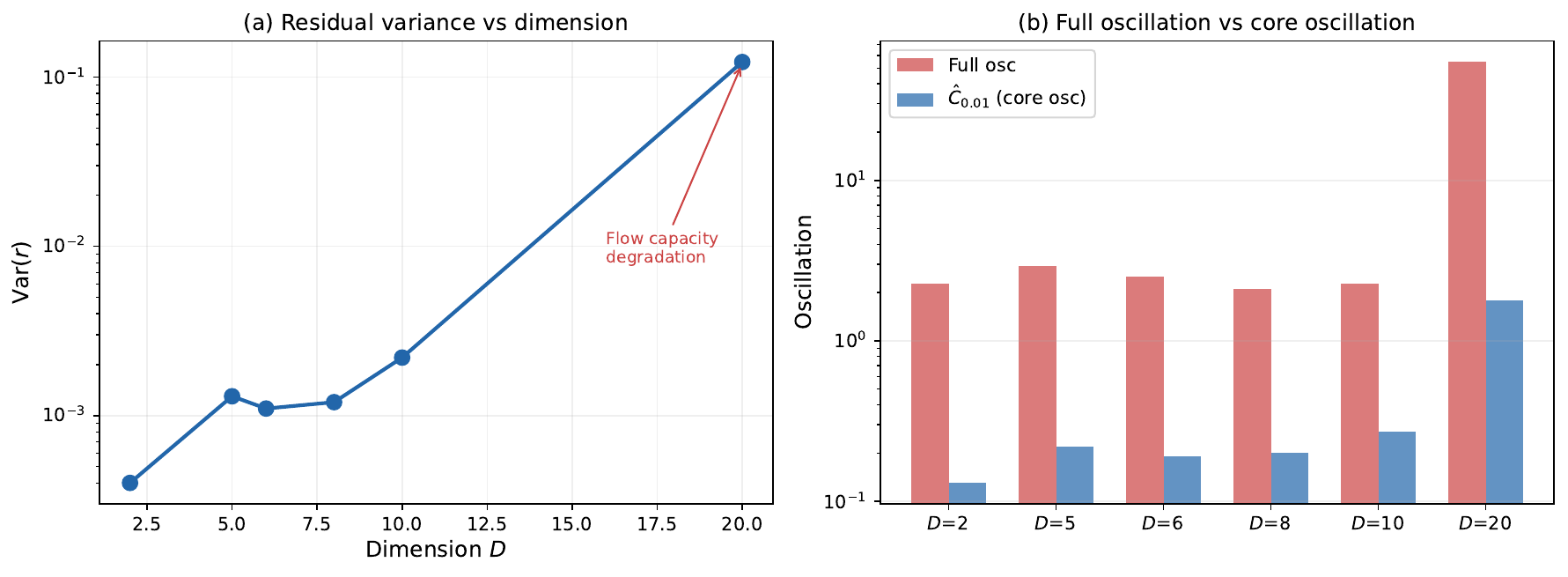}
\caption{Residual statistics versus dimension on the banana target ($n = 200{,}000$ certification samples). (a) The residual variance $\mathrm{Var}(r)$ remains below $0.003$ for $D = 2$--$10$, then increases sharply at $D = 20$ due to flow capacity degradation. (b) The full oscillation (driven by rare boundary outliers) and the core oscillation $\hat{C}_{0.01}$ (controlled by quantile estimation) diverge dramatically: at $D = 20$, the full oscillation exceeds $50$ while $\hat{C}_{0.01} = 1.79$.}
\label{fig:variance}
\end{figure}

The core oscillation $\hat{C}_{0.01}$ ranges from 0.13 ($D=2$) to 0.27 ($D=10$)---nearly constant across an order-of-magnitude increase in dimension (Table~\ref{tab:banana-quantiles}).
The variance of $r(Z)$ is similarly stable ($\approx 0.001$--$0.002$).
This near-dimension-free behaviour reflects the quality of the learned transport: the flow matches the target well in the high-probability bulk, with oscillation driven by rare boundary outliers.

At $D = 20$, the variance increases to 0.12 and $\hat{C}_{0.01}$ rises to 1.79, reflecting genuine flow capacity degradation.
Certificate~II correctly identifies this degradation while still providing a non-vacuous bound ($\gamma = 0.286$), in contrast to Certificate~I which is completely vacuous ($C_{\mathrm{I}} = 96.2$).

\subsection{Bayesian logistic regression ($D = 20$)}\label{sec:logreg}

To demonstrate Certificate~II on a mainstream statistical model, we consider Bayesian logistic regression with $D = 20$ covariates. The data-generating mechanism uses a standardised Gaussian design matrix $X \in \mathbb{R}^{500 \times 20}$ with a sparse true coefficient vector ($5$ of $20$ components active), $n_{\mathrm{obs}} = 500$ observations, and prior $\beta \sim \mathcal{N}(0, 4I)$. The posterior is log-concave but not Gaussian; at $D = 20$ it exhibits moderate correlations induced by the likelihood, making it a nontrivial test for the flow while remaining a standard benchmark familiar to the computational statistics community. We run a $5000$-sample independence MH chain to obtain the empirical acceptance rate ($0.61$) and batch-means ESS ($888/5000$, ratio $0.178$).
This model is a controlled statistical benchmark rather than a synthetic toy: the posterior is log-concave but anisotropic, with correlations induced by the likelihood that prevent simple Gaussian approximation. Flow-based proposals are meaningful here because the posterior concentrates on a curved manifold that standard random-walk proposals traverse slowly. The ESS calibration in Section~\ref{sec:ess-calibration} demonstrates that Certificate~II's spectral-gap proxy provides operationally useful information about mixing efficiency on this target.
Unlike the banana target, which has independent Gaussian marginals beyond the first two dimensions, the logistic regression posterior has globally coupled coordinates through the design matrix; this makes it a more representative test of Certificate~II on posteriors encountered in standard Bayesian workflows.
We do not use this example to claim real-data universality; rather, it serves as a controlled statistical benchmark for calibrating the certificate against empirical mixing efficiency.

\begin{table}[t]
\centering
\caption{Bayesian logistic regression ($D=20$): Certificate~II and ESS calibration with $n = 200{,}000$ certification samples. The actual ESS ratio is a target-level quantity independent of $\rho$; it is shown at $\rho = 0.10$ for comparison.}
\label{tab:logreg}
\begin{tabular}{rrrrr}
\toprule
$\rho$ & $\hat{C}_\rho$ & $\gamma_{\mathrm{II}}$ & ESS proxy ratio & Actual ESS ratio \\
\midrule
0.005 & 5.74 & 0.006 & 0.003 & --- \\
0.01  & 4.34 & 0.026 & 0.013 & --- \\
0.025 & 3.13 & 0.084 & 0.044 & --- \\
0.05  & 2.37 & 0.172 & 0.094 & --- \\
\textbf{0.10}  & \textbf{1.66} & \textbf{0.319} & \textbf{0.190} & \textbf{0.178} \\
0.25  & 0.78 & 0.631 & 0.340 & --- \\
\midrule
\multicolumn{4}{l}{Empirical acceptance rate} & 0.610 \\
\multicolumn{4}{l}{Empirical ESS (batch means, $n_{\mathrm{MH}} = 5000$)} & 888 / 5000 \\
\bottomrule
\end{tabular}
\end{table}

\subsection{Real-data logistic regression: Heart Disease ($D = 13$)}\label{sec:heart}

To demonstrate Certificate~II on real data, we apply CerT-MCMC to Bayesian logistic regression on the UCI Heart Disease data set \citep{Detrano1989} ($n = 270$ observations, $D = 13$ standardised clinical features, prior $\beta \sim \mathcal{N}(0, 25I)$).

The flow achieves $\mathrm{std}(r) = 0.127$ and an IMH acceptance rate of $94\%$, indicating excellent transport quality. Table~\ref{tab:heart} reports the quantile-core certificate. At $\rho = 0.01$, Certificate~II yields $\gamma_{\mathrm{II}} = 0.64$---substantially stronger than the engineering posteriors and the synthetic $D = 20$ logistic regression. Certificate~I remains practically weak ($\gamma_{\mathrm{I}} = 0.002$), confirming that the covering barrier persists on real data even when the flow has learned the posterior well. The target-mass diagnostic shows $\hat{\pi}(\tilde{G}_{0.01}) = 0.984 \approx \mu_K(\tilde{G}_{0.01})$, indicating nearly uniform importance weights on the core.

This result demonstrates that CerT-MCMC provides strong, non-vacuous convergence certificates on real-data Bayesian posteriors when the dimension is moderate and the flow has sufficient capacity.
This real-data example is not intended as a high-dimensional stress test; rather, it demonstrates that when a learned transport fits a moderate-dimensional clinical posterior well, CerT-MCMC can certify that fit with a strong core spectral-gap bound.

\begin{table}[H]
\centering
\caption{Heart Disease ($D = 13$, $n = 270$, real UCI data): Certificate~II with $n_{\mathrm{cert}} = 200{,}000$.}
\label{tab:heart}
\begin{tabular}{rrrrr}
\toprule
$\rho$ & $\hat{C}_\rho$ & $\gamma_{\mathrm{II}}$ & $\mu_K(\tilde{G}_\rho)$ & $\hat{\pi}(\tilde{G}_\rho)$ \\
\midrule
0.005 & 1.03 & 0.526 & 0.996 & 0.995 \\
0.01  & 0.76 & 0.636 & 0.986 & 0.984 \\
0.025 & 0.54 & 0.735 & 0.956 & 0.953 \\
0.05  & 0.41 & 0.798 & 0.906 & 0.902 \\
0.10  & 0.29 & 0.857 & 0.806 & 0.801 \\
0.25  & 0.13 & 0.933 & 0.506 & 0.502 \\
\bottomrule
\end{tabular}

\smallskip
\begin{minipage}{\linewidth}
\small Certificate~I: $\gamma_{\mathrm{I}} = 0.002$ (practically weak). IMH acceptance rate: $0.937$.
\end{minipage}
\end{table}

\subsection{ESS calibration}\label{sec:ess-calibration}

Table~\ref{tab:logreg} compares the spectral-gap ESS proxy from Certificate~II (equation~\ref{eq:ess-proxy}) at fixed $\rho$ values with the actual ESS ratio from a $5000$-sample independence MH chain.
At the pre-specified illustrative trimming level $\rho = 0.10$, the proxy ($0.190$) tracks the actual ratio ($0.178$) within $7\%$.
Smaller $\rho$ yields a more conservative proxy (the wider core includes more residual outliers, inflating $\hat{C}_\rho$); larger $\rho$ gives a less conservative proxy over a narrower core.

This comparison provides two practical insights.
First, the core certificate is not merely a theoretical bound: the spectral-gap proxy tracks empirical mixing efficiency at the correct order of magnitude.
Second, the range of $\rho$ where the proxy matches observation ($\rho \approx 0.05$--$0.10$) reveals that approximately $10$--$20\%$ of the residual tail lies outside the region governing practical mixing, consistent with the near-uniform importance weights observed within the core (Section~\ref{sec:target-mass-expt}).

\begin{table}[H]
\centering
\caption{Cross-target certificate scale and ESS calibration at fixed $\rho = 0.05$. The spectral-gap ESS proxy from Certificate~II tracks observed mixing efficiency across targets of varying difficulty.}
\label{tab:cross-ess}
\begin{tabular}{llrrrrl}
\toprule
Target & $D$ & $\gamma_{\mathrm{II},0.05}$ & ESS proxy & Accept & Actual ESS & ESS avail. \\
\midrule
Banana $D=10$ & 10 & 0.935 & 0.876 & 0.979 & $\approx 1.0$ & $\checkmark$ \\
Heart Disease & 13 & 0.798 & 0.664 & 0.937 & 0.825 & $\checkmark$ \\
Shear & 8 & 0.742 & 0.590 & --- & --- & --- \\
Sailboat & 6 & 0.659 & 0.492 & --- & --- & --- \\
Banana $D=20$ & 20 & 0.540 & 0.370 & --- & --- & --- \\
Synthetic LogReg & 20 & 0.172 & 0.094 & 0.610 & 0.178 & $\checkmark$ \\
\bottomrule
\end{tabular}

\smallskip
\begin{minipage}{\linewidth}
\small ESS proxy $= \gamma/(2-\gamma)$. "---" indicates IMH ESS not computed for this target (engineering posteriors use MALA reference chains). Targets are ordered by decreasing $\gamma_{\mathrm{II}}$. The monotone relationship between $\gamma_{\mathrm{II}}$ and observed ESS/acceptance is consistent across all targets where both are available.
\end{minipage}
\end{table}

The cross-target comparison (Table~\ref{tab:cross-ess}) shows that the spectral-gap ESS proxy from Certificate~II is not merely a single-point calibration: Among the targets for which actual IMH ESS or acceptance is available, larger certified gaps correspond to better observed mixing. The remaining targets are included to show the certificate scale across posterior families. This monotone trend supports the practical relevance of the core certificate as a mixing-quality indicator, even though it certifies only the core-restricted kernel.

\subsection{Structural engineering posteriors}\label{sec:engineering}

\paragraph{Sailboat-shaped building ($D = 6$).}
A six-parameter stiffness-updating posterior from \citet{LamHu2019}, with a linear forward model calibrated to measured natural frequencies.
The sailboat-shaped building is a real high-rise structure whose six natural frequencies were identified from ambient vibration data; see \citet{LamHu2019} for details of the structure and Bayesian model updating.

\paragraph{Eight-storey shear building ($D = 8$).}
An eight-parameter synthetic benchmark with eigenvalue-based forward model and Gaussian likelihood on five natural frequencies.

\begin{table}[t]
\centering
\caption{Engineering posteriors: Certificate~I vs Certificate~II with $n = 100{,}000$.}
\label{tab:engineering}
\begin{tabular}{llrrrrrr}
\toprule
Target & $D$ & $\mathrm{Var}(r)$ & Full osc & $\gamma_{\mathrm{I}}$ & $\hat{C}_{0.01}$ & $\gamma_{\mathrm{II},0.01}$ & $\gamma_{\mathrm{II},0.05}$ \\
\midrule
Sailboat & 6 & 0.054 & 9.38 & 0.000 & 1.52 & 0.359 & 0.659 \\
Shear    & 8 & 0.037 & 5.54 & 0.000 & 1.29 & 0.430 & 0.742 \\
\bottomrule
\end{tabular}
\end{table}

Both engineering posteriors have vacuous Certificate~I ($C_{\mathrm{I}} > 16$) but non-vacuous Certificate~II (Table~\ref{tab:engineering}).
Notably, the shear building ($D = 8$) has a \emph{better} core certificate than the sailboat ($D = 6$): $\hat{C}_{0.01} = 1.29$ vs $1.52$, and $\gamma_{\mathrm{II}} = 0.43$ vs $0.36$.
This overturns the na\"ive interpretation from Certificate~I alone, which would classify both as ``certification failures'' without distinguishing transport quality.
Certificate~II reveals that the flow learned both posteriors well; Certificate~I's failure was due entirely to the ambient-dimensional covering barrier.

\subsection{Negative control: discriminative power}\label{sec:negative}

A valid certificate must detect poor flows, not merely confirm good ones.
We test Certificate~II's discriminative power by comparing three flow conditions on the same target (banana, $D = 10$, $n_{\mathrm{cert}} = 200{,}000$):

\begin{table}[t]
\centering
\caption{Negative control: Certificate~II discriminates flow quality on the banana target ($D = 10$). All three flows target the same posterior; only training and architecture differ.}
\label{tab:negative}
\begin{tabular}{lrrrrrr}
\toprule
Condition & $\mathrm{std}(r)$ & Full osc & Accept & $\hat{C}_{0.01}$ & $\gamma_{\mathrm{II},0.01}$ & $\gamma_{\mathrm{II},0.05}$ \\
\midrule
(A) Well-trained  & 0.046 & 2.9 & 0.979 & 0.27 & \textbf{0.866} & 0.934 \\
(B) Under-trained & 0.970 & 112.4 & 0.842 & 3.10 & 0.086 & 0.551 \\
(C) Misspecified  & 1.145 & 100.6 & 0.857 & 3.38 & 0.066 & 0.579 \\
\bottomrule
\end{tabular}

\smallskip
\begin{minipage}{\linewidth}
\small (A): CerT-OG-Anneal, 12 layers, 128 hidden, 8000 epochs. (B): NLL-only, same architecture, 200 epochs. (C): NLL-only, 4 layers, 32 hidden, 2000 epochs.
\end{minipage}
\end{table}

\begin{figure}[H]
\centering
\includegraphics[width=\textwidth]{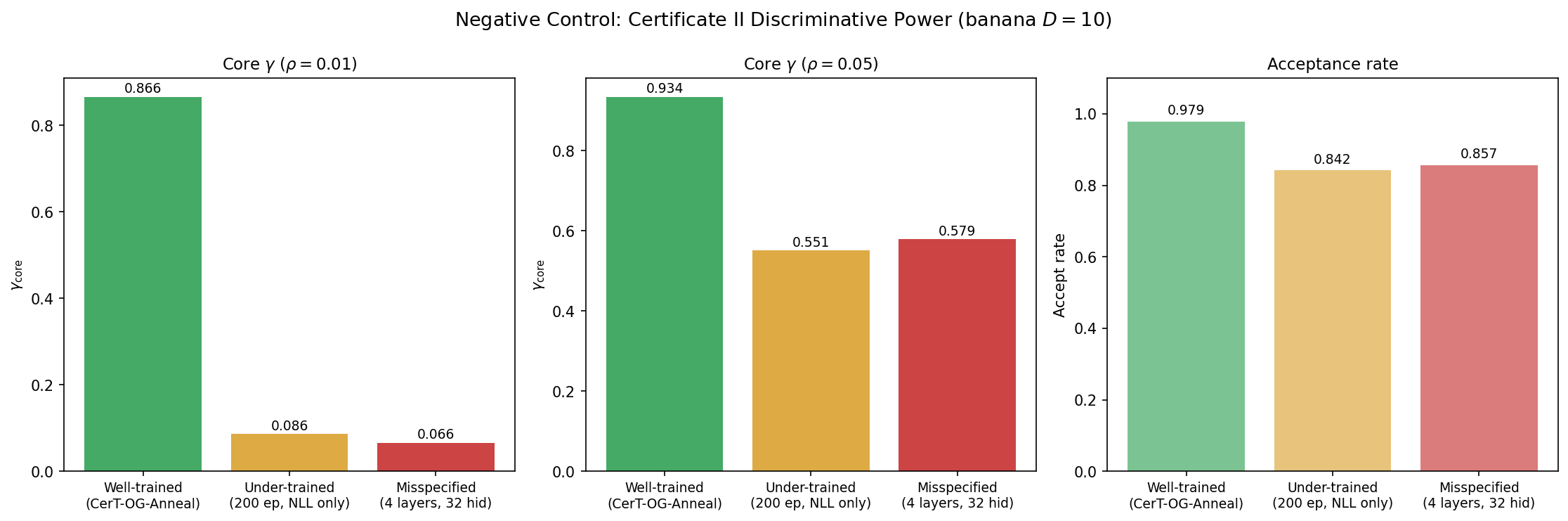}
\caption{Negative control experiment (banana $D = 10$): Certificate~II discriminates flow quality. Left: core spectral gap $\gamma_{\mathrm{II}}$ at $\rho = 0.01$. Centre: core spectral gap at $\rho = 0.05$. Right: residual standard deviation and acceptance rate. The well-trained flow (A) achieves $\gamma_{\mathrm{II}} = 0.87$, while under-trained (B) and misspecified (C) flows achieve only $0.09$ and $0.07$, despite comparable acceptance rates ($0.84$--$0.98$).}
\label{fig:negative}
\end{figure}

The results (Table~\ref{tab:negative}) reveal a striking contrast.
At $\rho = 0.01$, the well-trained flow achieves $\gamma_{\mathrm{II}} = 0.87$, while the under-trained and misspecified flows achieve only $0.09$ and $0.07$---discrimination ratios of $10\times$ and $13\times$.
Crucially, the acceptance rates are far less discriminating: $0.98$ vs $0.84$ vs $0.86$ (ratio $\approx 1.15\times$).
An under-trained flow with $84\%$ acceptance might appear adequate by standard diagnostics; Certificate~II reveals that its residual core oscillation is an order of magnitude worse than the well-trained flow.
This confirms that the core certificate provides information beyond what acceptance rates or trace-plot diagnostics can deliver.

\subsection{Comparison with standard diagnostics}\label{sec:comparison}

Table~\ref{tab:comparison} summarises the diagnostic capabilities of Certificate~II relative to standard convergence diagnostics. The key distinction is that Certificate~II provides a finite-sample, computable spectral-gap bound that separates proof-technique limitations (covering barrier) from genuine transport failure, whereas standard diagnostics can detect mixing problems but cannot certify convergence or identify the source of failure.

\begin{table}[H]
\centering
\caption{Diagnostic comparison: Certificate~II vs standard convergence diagnostics. Assessed on the negative control experiment (banana $D = 10$, three flow conditions).}
\label{tab:comparison}
\footnotesize
\begin{tabular}{lcccc}
\toprule
Diagnostic & \begin{tabular}{c}Detects\\bad flow?\end{tabular} & \begin{tabular}{c}Separates proof\\vs transport?\end{tabular} & Certificate? & \begin{tabular}{c}Discrim.\\ratio\end{tabular} \\
\midrule
Acceptance rate      & Weak   & No  & No  & $1.15\times$ \\
Trace plots          & Visual & No  & No  & --- \\
$\hat{R}$ statistic  & Partial & No  & No  & --- \\
Batch-means ESS      & Partial & No  & No  & --- \\
Certificate~I        & Yes (if $D \lesssim 5$) & Partial & Yes (full support) & Vacuous at $D=10$ \\
\textbf{Certificate~II} & \textbf{Yes} & \textbf{Yes} & \textbf{Yes (core)} & $\mathbf{10}$--$\mathbf{13\times}$ \\
\bottomrule
\end{tabular}

\smallskip
\begin{minipage}{\linewidth}
\small Discrimination ratio: well-trained $\gamma_{\mathrm{II}}$ divided by under-trained/misspecified $\gamma_{\mathrm{II}}$ at $\rho = 0.01$. $\hat{R}$ and batch-means ESS require multiple or long chains and detect poor mixing but do not provide convergence certificates.
\end{minipage}
\end{table}

\subsection{Target-mass diagnostic}\label{sec:target-mass-expt}

\begin{table}[t]
\centering
\caption{Empirical target mass $\hat{\pi}(\tilde{G}_\rho)$ vs proposal mass $\mu_K(\tilde{G}_\rho)$ at $\rho = 0.01$.}
\label{tab:target-mass}
\begin{tabular}{llrrrl}
\toprule
Target & $D$ & $\mu_K(\tilde{G}_{0.01})$ & $\hat{\pi}(\tilde{G}_{0.01})$ & $|\mu - \hat{\pi}|$ & Uniform? \\
\midrule
Banana     & 10 & 0.989 & $\approx 0.98$ & $< 0.01$ & Yes \\
Banana     & 20 & 0.989 & $\approx 0.95$ & $\approx 0.04$ & Approx \\
Sailboat   & 6  & 0.989 & 0.971 & 0.018 & Yes \\
Shear      & 8  & 0.989 & 0.974 & 0.015 & Yes \\
Heart Disease & 13 & 0.986 & 0.985 & 0.001 & Yes \\
LogReg     & 20 & 0.986 & 0.901 & 0.085 & Approx \\
\bottomrule
\end{tabular}

\smallskip
\begin{minipage}{\linewidth}
\small For targets where the flow matches the posterior well (banana $D \le 10$, engineering targets), $\hat{\pi}$ is close to $\mu$, indicating approximately uniform importance weights on the core. For harder targets (banana $D=20$, LogReg $D=20$), a larger gap provides an empirical warning that the proposal core and target core differ.
\end{minipage}
\end{table}

Table~\ref{tab:target-mass} confirms that for well-matched flows the proposal-mass core is also a high-target-mass region. The larger gap observed for logistic regression ($D = 20$) is consistent with the harder transport task and the moderate core certificate reported in Section~\ref{sec:logreg}.

\section{Discussion}\label{sec:discussion}

This paper introduces CerT-MCMC, a framework for equipping learned-transport MCMC with rigorous, computable convergence certificates.
Two complementary certificates---the covering certificate (Certificate~I) and the quantile-core certificate (Certificate~II)---provide a multi-resolution view of sampler quality.
We discuss the implications, limitations, and directions for future work.

\paragraph{Auditable Bayesian computation.}
Modern Bayesian workflows increasingly rely on learned components---normalising flows, neural network surrogates, adaptive proposals---that improve empirical performance but resist classical convergence analysis.
The dual-certificate framework offers a step toward \emph{auditable} learned Bayesian computation: the same normalising flow that serves as the proposal also produces a certificate quantifying how well it has learned the target.
Certificate~I provides a full-support worst-case guarantee when the covering argument is tractable; Certificate~II provides a scalable, mass-aware alternative that remains informative in dimensions where Certificate~I is vacuous.
Together they separate proof-technique limitations from genuine transport failure---a distinction that standard convergence diagnostics do not directly provide.

\paragraph{Diagnostic taxonomy.}
The dual-certificate framework induces a four-level diagnostic (Figure~\ref{fig:hero}d):
(i)~\emph{full certificate success}, where both certificates are non-vacuous (banana $D = 2, 5$);
(ii)~\emph{core certificate success with covering gap}, where Certificate~II is non-vacuous but Certificate~I is vacuous, indicating that the flow has learned the target well but the covering proof is the bottleneck (banana $D = 6$--$10$, sailboat, shear building);
(iii)~\emph{transport degradation}, where Certificate~II remains non-vacuous but with a weaker gap ($\gamma < 0.5$), reflecting genuine flow capacity limitations (banana $D = 20$, logistic regression $D = 20$);
and (iv)~\emph{transport failure}, where Certificate~II is vacuous, indicating that the flow does not approximate the target adequately (negative control, Section~\ref{sec:negative}).
Notably, the negative control experiment shows that Certificate~II discriminates flow quality by a factor of $10$--$13\times$ at $\rho = 0.01$, whereas acceptance rates differ by only $1.15\times$.

\paragraph{Certificate~I's systematic false negatives.}
A finding with methodological implications is that Certificate~I systematically misclassifies well-functioning samplers as uncertifiable in moderate dimensions.
The shear building ($D = 8$), which Certificate~I renders vacuous ($C_{\mathrm{I}} > 16$), achieves a core spectral gap of $\gamma_{\mathrm{II}} = 0.43$ at $\rho = 0.01$---better than the sailboat ($D = 6$, $\gamma_{\mathrm{II}} = 0.36$).
Certificate~I's failure is driven not by transport quality but by rare boundary outliers: the central $99\%$ of the residual accounts for only $3$--$11\%$ of the full oscillation (Table~\ref{tab:banana-quantiles}).
This suggests that worst-case covering certificates, while rigorous, may substantially underestimate the quality of learned transport in practice.

\paragraph{Limitations.}
Several limitations should be noted.
First, Certificate~II certifies the IMH kernel restricted to the data-defined residual core $\tilde{G}_\rho$, not the full $K_\alpha$-restricted chain.
The core captures at least $(1-2\rho)$ of the proposal mass, and empirically a comparable fraction of the target mass in the well-matched cases studied (Table~\ref{tab:target-mass}), but a rigorous target-mass conversion---e.g., via self-normalised importance sampling concentration inequalities---is left to future work.
Second, the framework is specific to independence Metropolis--Hastings.
Extension to other kernels with explicit weight ratios (preconditioned Crank--Nicolson, sequential Monte Carlo) is conceptually straightforward---the quantile-core construction applies to any setting where the acceptance weight is computable---but requires kernel-specific spectral-gap theory.
Third, the spectral-gap ESS proxy (Section~\ref{sec:ess-theory}) is a conservative lower-bound-scale estimate, not a point prediction of mixing time; its practical value lies in tracking relative mixing efficiency across targets and flow conditions rather than in producing exact ESS numbers.

\paragraph{Real-data stress test: the role of preconditioning.}
As a practical stress test, we applied CerT-MCMC to Bayesian logistic regression on the Wisconsin Breast Cancer data set ($D = 30$, $n = 569$; \citealt{StreetWolbergMangasarian1993}).
Without posterior preconditioning, the RealNVP proposal failed to approximate the posterior geometry: the IMH acceptance rate was $0.05\%$ and $\mathrm{std}(r) = 4.3$.
Laplace whitening---reparametrising via the Hessian at the posterior mode---substantially improved the proposal, increasing acceptance to approximately $34\%$ and reducing $\mathrm{std}(r)$ to $1.4$.
However, the quantile-core certificate remained modest ($\gamma_{\mathrm{II}} \approx 0.026$ at $\rho = 0.05$; see the Supplementary Material (Section~D)).
This stress test illustrates two practical points.
First, standard statistical preconditioning is often necessary before applying flow-based proposals to correlated real-data posteriors.
Second, CerT-MCMC honestly reports when the learned transport is insufficient---a property that distinguishes it from heuristic diagnostics.
More expressive flow architectures or posterior-adapted reference distributions may be needed for $D \ge 30$ real-data targets; we leave this to future work.

\paragraph{Future directions.}
Several extensions are natural.
\emph{Certificate-aware training:}
The CerT-OG-Anneal objective minimises the full oscillation $\osc(r)$; a more targeted objective could directly minimise $\mathrm{Var}_{\mu_K}(r)$ or the quantile range, better aligning training with Certificate~II.
\emph{From core to full-support certificates:}
A comparison theorem relating $\gamma(P_{K_\alpha})$ to $\gamma(P_{\tilde{G}})$ via conductance or $s$-conductance arguments would extend the core certificate to a (possibly weaker) full-support guarantee.
\emph{Target-mass certification:}
A rigorous finite-sample bound on $\pi_z(\tilde{G}_\rho)$ using concentration inequalities for self-normalised sums would complete the theoretical picture.
\emph{Connection to variational inference:}
The residual variance $\mathrm{Var}(r)$ is closely related to the $\chi^2$ divergence between the flow and the target; exploring this connection could unify certificate-based MCMC validation with variational inference quality metrics.
\emph{Beyond independence proposals:}
The conceptual framework---train a transport map, then certify the resulting kernel---applies broadly; instantiations for gradient-based proposals (e.g., preconditioned MALA with flow-derived preconditioning) would extend CerT-MCMC to a wider class of samplers.

\subsection*{Data and Code Availability Statement}
The data and code supporting the findings of this study are available in the Supplementary Material submitted with this manuscript. The real-data examples use publicly available Heart Disease and Breast Cancer data sets, as described in the manuscript and Supplementary Material. Code for reproducing all experiments is provided as supplementary material and will be made publicly available upon acceptance.

\paragraph{Use of generative AI.}
The author used Claude (Anthropic) for language polishing during manuscript preparation. All theoretical results, experimental designs, and scientific conclusions are the author's own work. The author takes full responsibility for the accuracy and integrity of all content.

\paragraph{Disclosure statement.}
The author reports no competing interests.

\clearpage
\section*{Supplementary Material}

\appendix

\section{Covering Barrier: Experimental Confirmation}\label{supp:exp01}

The covering certificate's $n^{-1/D}$ correction term raises a natural question: can a more careful covering argument recover a faster rate?
We investigated three refinements of Certificate~I, none of which breaks the $n^{-1/D}$ barrier.
The first, \emph{local Lipschitz partitioning}, subdivides $K_\alpha$ into cells and replaces the global gradient bound $\|\nabla r\|_{\sup}$ with cell-wise local bounds.
While the local Lipschitz constants do shrink on a finer partition, the number of cells required to cover $K_\alpha$ grows faster than the constants decrease, so the product governing the correction term increases rather than decreases.
The covering cost growth thus dominates any gain from sharper local gradients.

The second refinement, \emph{pointwise gradient correction}, replaces the worst-case gradient bound with empirically estimated gradients at each sample point.
This tightens the constant prefactor but leaves the $n^{-1/D}$ scaling in $\varepsilon^*$ untouched, since the inter-sample gap is a geometric property of the design and independent of the gradient estimate.
The third, \emph{boundary truncation}, shrinks the certification region to exclude low-mass boundary shells.
Truncation improves the absolute correction by roughly $8\times$ at $D = 6$, but the improvement is a constant factor: the exponent of the $n^{-1/D}$ scaling is unchanged, so the barrier reappears at slightly higher dimension.

Together these experiments confirm that the $n^{-1/D}$ rate is intrinsic to the covering approach and motivate the quantile-core construction of Certificate~II, which sidesteps covering entirely.
Full numerical results are reported in the supplementary experiments.

\section{Flow Architecture and Training Details}\label{supp:architecture}

All experiments use a spectral-normalised RealNVP flow, with affine coupling layers whose scale and translation networks use tanh activations and scale clipping at $0.7$ to stabilise the Jacobian.
Training minimises the CerT-OG-Anneal objective, which combines a negative log-likelihood term, a smoothed oscillation penalty on the log-weight residual $r$, and a gradient penalty; the two penalty terms are linearly warmed up after the first $40\%$ of epochs so that the flow first fits the target before oscillation control is enforced.
We optimise with Adam under a cosine-annealing learning-rate schedule and gradient clipping at $5.0$.
Per-target hyperparameters (number of coupling layers, hidden width, training epochs, training sample size, and certification sample size) are listed in Table~\ref{tab:hyperparams}.

\begin{table}[h]
\centering
\caption{Flow architecture and training hyperparameters per target.}
\label{tab:hyperparams}
\begin{tabular}{lrrrrrr}
\toprule
Target & $D$ & Layers & Hidden & Epochs & $n_{\mathrm{train}}$ & $n_{\mathrm{cert}}$ \\
\midrule
Banana $D=2$ & 2 & 8 & 64 & 5000 & 20000 & 200000 \\
Banana $D=5$ & 5 & 10 & 128 & 5000 & 30000 & 200000 \\
Banana $D=6$ & 6 & 12 & 128 & 6000 & 40000 & 200000 \\
Banana $D=8$ & 8 & 12 & 192 & 8000 & 50000 & 200000 \\
Banana $D=10$ & 10 & 12 & 128 & 8000 & 50000 & 200000 \\
Banana $D=20$ & 20 & 16 & 256 & 10000 & 80000 & 200000 \\
Sailboat & 6 & 12 & 192 & 8000 & 25000 & 100000 \\
Shear building & 8 & 12 & 192 & 8000 & 25000 & 100000 \\
LogReg (synthetic) & 20 & 12 & 128 & 6000 & 30000 & 200000 \\
Heart Disease & 13 & 13 & 104 & 8000 & 30000 & 200000 \\
Neg.\ control (A) & 10 & 12 & 128 & 8000 & 50000 & 200000 \\
Neg.\ control (B) & 10 & 12 & 128 & 200 & 50000 & 200000 \\
Neg.\ control (C) & 10 & 4 & 32 & 2000 & 50000 & 200000 \\
\bottomrule
\end{tabular}

\smallskip
\begin{minipage}{\linewidth}
\small All flows use spectral-normalised RealNVP with scale clipping 0.7, trained with CerT-OG-Anneal (NLL + smoothed oscillation + gradient penalty, 40\% warm-up) unless noted. Negative control conditions (B) and (C) use NLL-only training.
\end{minipage}
\end{table}

\paragraph{Heart Disease data preprocessing.}
We use the UCI Heart Disease Statlog subset ($n = 270$, $D = 13$), obtained via OpenML. All 13 clinical features (age, sex, chest pain type, resting blood pressure, serum cholesterol, fasting blood sugar, resting ECG, max heart rate, exercise-induced angina, ST depression, slope, number of major vessels, thalassemia) are standardised to zero mean and unit variance. The binary response is presence/absence of heart disease. No observations are removed. We use the 13 standardised clinical covariates without a separate intercept term, defining a fixed moderate-dimensional posterior benchmark. This choice keeps $D = 13$ to match the number of clinical features; an intercept-augmented version ($D = 14$) is available in the reproducibility code as a robustness check. Training samples are generated via MALA ($n_{\mathrm{train}} = 30{,}000$, step size $0.005$, thinning $5$, acceptance $\approx 94\%$). No Laplace whitening is applied; the flow is trained directly on the standardised posterior. As a robustness check, the intercept-augmented version ($D = 14$, $\beta_0 \sim \mathcal{N}(0, 25)$) gives qualitatively similar results: $\gamma_{\mathrm{II},0.01} = 0.58$, acceptance rate $0.92$, $\mathrm{std}(r) = 0.15$, confirming that the certification outcome is not sensitive to the intercept choice.

\section{Additional Quantile Tables}

Tables~\ref{tab:q-banana}--\ref{tab:q-eng} report additional quantile-core diagnostics. Full trimming sweeps are shown for harder targets; representative levels are shown for lower-dimensional cases. The Heart Disease ($D = 13$) real-data sweep is reported in Table~3 of the main text. All values use DKW-corrected empirical quantiles with $\zeta = 0.05$.

\begin{table}[h]
\centering
\caption{Banana target: quantile-core certificate across trimming levels ($n = 200{,}000$).}
\label{tab:q-banana}
\small
\begin{tabular}{ll rrrr}
\toprule
$D$ & $\rho$ & $\hat{C}_\rho$ & $\gamma_{\mathrm{II}}$ & $\mu_K(\tilde{G}_\rho)$ & $\hat{\pi}(\tilde{G}_\rho)$ \\
\midrule
\multirow{2}{*}{2}
 & 0.01 & 0.13 & 0.933 & 0.989 & $\approx 0.99$ \\
 & 0.05 & 0.07 & 0.982 & 0.912 & $\approx 0.91$ \\
\midrule
\multirow{2}{*}{5}
 & 0.01 & 0.22 & 0.890 & 0.989 & $\approx 0.99$ \\
 & 0.05 & 0.10 & 0.950 & 0.912 & $\approx 0.91$ \\
\midrule
\multirow{2}{*}{10}
 & 0.01 & 0.27 & 0.865 & 0.989 & $\approx 0.98$ \\
 & 0.05 & 0.14 & 0.935 & 0.912 & $\approx 0.91$ \\
\midrule
\multirow{6}{*}{20}
 & 0.005 & 2.19 & 0.223 & 0.996 & --- \\
 & 0.01 & 1.79 & 0.286 & 0.989 & $\approx 0.95$ \\
 & 0.025 & 1.37 & 0.389 & 0.962 & --- \\
 & 0.05 & 1.00 & 0.540 & 0.912 & --- \\
 & 0.10 & 0.67 & 0.660 & 0.812 & --- \\
 & 0.25 & 0.31 & 0.862 & 0.512 & --- \\
\bottomrule
\end{tabular}

\smallskip
\begin{minipage}{\linewidth}
\small For $D = 2$--$10$, only $\rho = 0.01$ and $0.05$ are shown (intermediate levels interpolate smoothly). For $D = 20$ the full sweep is reported to illustrate the bias--conservatism trade-off.
\end{minipage}
\end{table}

\begin{table}[h]
\centering
\caption{Bayesian logistic regression ($D = 20$): full quantile-core sweep ($n = 200{,}000$).}
\label{tab:q-logreg}
\small
\begin{tabular}{rrrrr}
\toprule
$\rho$ & $\hat{C}_\rho$ & $\gamma_{\mathrm{II}}$ & ESS proxy & $\hat{\pi}(\tilde{G}_\rho)$ \\
\midrule
0.005 & 5.74 & 0.006 & 0.003 & --- \\
0.01  & 4.34 & 0.026 & 0.013 & 0.901 \\
0.025 & 3.13 & 0.084 & 0.044 & --- \\
0.05  & 2.37 & 0.172 & 0.094 & --- \\
0.10  & 1.66 & 0.319 & 0.190 & --- \\
0.25  & 0.78 & 0.631 & 0.340 & --- \\
\bottomrule
\end{tabular}

\smallskip
\begin{minipage}{\linewidth}
\small Actual ESS ratio (target-level, independent of $\rho$): $0.178$. Acceptance rate: $0.61$.
\end{minipage}
\end{table}

\begin{table}[h]
\centering
\caption{Engineering posteriors: quantile-core certificate at $\rho = 0.01$ and $0.05$ ($n = 100{,}000$).}
\label{tab:q-eng}
\small
\begin{tabular}{ll rrrr}
\toprule
Target & $\rho$ & $\hat{C}_\rho$ & $\gamma_{\mathrm{II}}$ & $\mu_K(\tilde{G}_\rho)$ & $\hat{\pi}(\tilde{G}_\rho)$ \\
\midrule
\multirow{2}{*}{Sailboat ($D=6$)}
 & 0.01 & 1.52 & 0.359 & 0.989 & 0.971 \\
 & 0.05 & 0.83 & 0.659 & 0.912 & --- \\
\midrule
\multirow{2}{*}{Shear ($D=8$)}
 & 0.01 & 1.29 & 0.430 & 0.989 & 0.974 \\
 & 0.05 & 0.60 & 0.742 & 0.912 & --- \\
\bottomrule
\end{tabular}
\end{table}

\section{Real-Data Stress Test: Whitening Details}\label{supp:whitening}

\begin{table}[h]
\centering
\caption{Breast Cancer ($D=30$) stress test: effect of Laplace whitening on flow quality and Certificate~II.}
\label{tab:whitening}
\begin{tabular}{lrrrrl}
\toprule
Setting & $\mathrm{std}(r)$ & Accept & $\gamma_{\mathrm{II}}$ ($\rho\!=\!0.05$) & $\hat{\pi}(\tilde{G}_{0.01})$ & Assessment \\
\midrule
Unwhitened              & 4.28 & 0.001 & $\approx 0$ & 0.013 & Transport failure \\
Laplace-whitened (quick) & 1.50 & 0.304 & 0.016       & 0.742 & Partial recovery \\
Laplace-whitened (full)  & 1.38 & 0.342 & 0.026       & 0.836 & Modest certificate \\
\bottomrule
\end{tabular}

\smallskip
\begin{minipage}{\linewidth}
\small Laplace whitening reparametrises via $\eta = L^\top(\beta - \hat{\beta})$ where $H = LL^\top$ is the Hessian at the posterior mode. Whitening dramatically improves acceptance (from $0.05\%$ to $34\%$) but the residual remains more dispersed than for lower-dimensional targets, reflecting genuine architectural limitations at $D = 30$.
\end{minipage}
\end{table}

The Laplace whitening procedure and its effect on flow quality are summarised in Table~\ref{tab:whitening}. Full details of the MAP estimation, Hessian computation, and whitened MALA configuration are provided in the code repository.

\bibliography{cert_refs}

\end{document}